\documentclass[10pt]{article}
\usepackage[top=1in, bottom=1in, left=0.9in, right=0.9in]{geometry}

\usepackage{zhiyuz}

\title{\textbf{Improving Adaptive Online Learning Using Refined Discretization}}
\author{
  \makebox[.3\textwidth]{Zhiyu Zhang} \\
  Harvard University\\
  \texttt{zhiyuz@seas.harvard.edu}\\
  \and
  \makebox[.3\textwidth]{Heng Yang} \\
  Harvard University\\
  \texttt{hankyang@seas.harvard.edu}\\
  \and
  \makebox[.3\textwidth]{Ashok Cutkosky} \\
  Boston University\\
  \texttt{ashok@cutkosky.com}\\
  \and
  \makebox[.3\textwidth]{Ioannis Ch. Paschalidis}\\
  Boston University\\
  \texttt{yannisp@bu.edu}\\
}
\date{}

\begin{document}
\maketitle

\begin{abstract}
We study unconstrained Online Linear Optimization with Lipschitz losses. Motivated by the pursuit of instance optimality, we propose a new algorithm that simultaneously achieves ($i$) the \textsc{AdaGrad}-style second order gradient adaptivity; and ($ii$) the comparator norm adaptivity also known as ``parameter freeness'' in the literature. In particular, 
\begin{itemize}
\item our algorithm does not employ the impractical doubling trick, and does not require an a priori estimate of the time-uniform Lipschitz constant;
\item the associated regret bound has the optimal $O(\sqrt{V_T})$ dependence on the gradient variance $V_T$, without the typical logarithmic multiplicative factor; 
\item the leading constant in the regret bound is ``almost'' optimal.
\end{itemize}

Central to these results is a continuous time approach to online learning. We first show that the aimed simultaneous adaptivity can be achieved fairly easily in a continuous time analogue of the problem, where the environment is modeled by an arbitrary continuous semimartingale. Then, our key innovation is a new discretization argument that preserves such adaptivity in the discrete time adversarial setting. This refines a non-gradient-adaptive discretization argument from \cite{harvey2023optimal}, both algorithmically and analytically, which could be of independent interest.
\end{abstract}

\section{Introduction}\label{section:intro}

We study unconstrained \emph{Online Linear Optimization} (OLO) with Lipschitz losses,\footnote{In a black-box manner, solutions of this problem can also solve bounded domain \emph{Online Convex Optimization} (OCO) with Lipschitz losses, as shown in \cite[Section~2.3]{orabona2023modern} and \cite[Section~4]{cutkosky2020parameter}. We start from the setting with a known Lipschitz constant $G$. The setting without knowing $G$ is considered in Section~\ref{section:lipschitz}.} which is a repeated game between us (the learner) and an adversarial environment denoted by $Env$. In each (the $t$-th) round, with a mutually known Lipschitz constant $G$:
\begin{enumerate}
\item We make a decision $x_t\in\R^d$ based on the observations before the $t$-th round.
\item The environment $Env$ reveals a loss gradient $g_t\in\R^d$ dependent on our decision history $x_1,\ldots,x_t$, which satisfies the Lipschitz condition with respect to the Euclidean norm, $\norm{g_t}\leq G$.
\item We suffer the linear loss $\inner{g_t}{x_t}$.
\end{enumerate}
The game ends after $T$ rounds, and then, our total loss is compared to that of an arbitrary fixed decision $u\in\R^d$. Without knowing the time horizon $T$, the environment $Env$ and the comparator $u$, our goal is to guarantee low \emph{regret}, defined as
\begin{equation*}
\reg_T(Env,u)\defeq\sum_{t=1}^T\inner{g_t}{x_t-u}.
\end{equation*}
In a nutshell, this paper proposes a novel and practical strategy to achieve the tightest known regret upper bound (even including a near-optimal leading constant) that depends simultaneously on the loss gradients $g_1,\ldots,g_T$ and the comparator $u$. 

To be concrete, we now introduce a bit more of the background. Existing research on this problem started from the minimax regime: under the additional assumption of $\norm{u}\leq D$, it has been long known that \emph{Online Gradient Descent} (OGD) \cite{zinkevich2003online} guarantees the optimal upper bound on the \emph{worst case regret}, $\sup_{Env;\norm{u}\leq D}\reg_T(Env,u)\leq O\rpar{DG\sqrt{T}}$. Refining such worst case optimality by instance optimality, improvements have been achieved under the notion of \emph{adaptive online learning}, with \emph{gradient adaptivity} and \emph{comparator adaptivity} being the two prominent types.
\begin{itemize}
\item Gradient adaptivity aims at bounding $\sup_{\norm{u}\leq D}\reg_T(Env,u)$ by a function of the observed gradient sequence $g_1,\ldots,g_T$. Using learning rates dependent on past observations, OGD can achieve the optimal \emph{second order} gradient adaptive bound \cite{mcmahan2010adaptive,duchi2011adaptive}
\begin{equation}\label{eq:gradient_adaptivity}
\sup_{\norm{u}\leq D}\reg_T(Env,u)\leq O\rpar{D\sqrt{V_T}},
\end{equation}
where $V_T\defeq\sum_{t=1}^T\norm{g_t}^2$ is the (uncentered) \emph{gradient variance}. This has been a hallmark of practical online learning algorithms, popularized by the massive success of \textsc{AdaGrad} \cite{duchi2011adaptive}.
\item Comparator adaptivity aims at bounding $\sup_{Env}\reg_T(Env,u)$ by a function of the comparator $u$. Without imposing the extra bounded-$u$ assumption, one could use a \emph{dual space framework} (which differs considerably from OGD) to achieve the optimal bound \cite{mcmahan2014unconstrained,zhang2022pde}
\begin{equation}\label{eq:comparator_adaptivity}
\sup_{Env}\reg_T(Env,u)\leq O\rpar{\norm{u}G\sqrt{T\log \norm{u}}}.
\end{equation}
Due to the absence of learning rates, such algorithms are also called ``parameter-free'' \cite{orabona2016coin} in the literature. They have exhibited the potential to reduce hyperparameter tuning in the modern deep learning workflow \cite{orabona2017training,cutkosky2023mechanic}. 
\end{itemize}

While both types of adaptivity are well-studied separately, achieving them simultaneously is an active research direction, which we call \emph{simultaneous adaptivity}. Again without the bounded-$u$ assumption, a series of works \cite{cutkosky2018black,mhammedi2020lipschitz,jacobsen2022parameter} have proposed drastically different approaches to obtain simultaneously adaptive regret bounds like\footnote{Omitting an additive $O\rpar{\norm{u}G\log(G^{-2}\norm{u}V_T)}$ term for clarity. Essentially, it means the order of $V_T$ is considered ``more important'' than the order of $\norm{u}$, which fits into the convention of the field.}
\begin{equation}\label{eq:simultaneous_old}
\reg_T(Env,u)\leq O\rpar{\norm{u}\sqrt{V_T\log (G^{-2}\norm{u} V_T)}},
\end{equation}
combining the strengths of Eq.(\ref{eq:gradient_adaptivity}) and Eq.(\ref{eq:comparator_adaptivity}) above. A remaining issue is that compared to the lower bound $O\rpar{\norm{u}\sqrt{V_T\log \norm{u}}}$ \cite{streeter2012no,orabona2013dimension,zhang2022pde}, Eq.(\ref{eq:simultaneous_old}) is still suboptimal due to a $\sqrt{\log V_T}$ multiplicative factor. That is, in terms of the dependence on the gradient variance $V_T$ alone (which is typically the emphasis of the field), Eq.(\ref{eq:simultaneous_old}) is $O\rpar{\sqrt{V_T\log V_T}}$ rather than the standard optimal rate $O\rpar{\sqrt{V_T}}$ as in Eq.(\ref{eq:gradient_adaptivity}). \textbf{The first goal of this paper, on the quantitative side, is to close this gap.}\footnote{\textcolor{red}{After the publication of this paper, we realized that the algorithm from \cite{jacobsen2022parameter} under a different hyperparameter setting not reported there can also achieve the $O\rpar{\norm{u}\sqrt{V_T\log \norm{u}}}$ bound. We refer the readers to Section~7 of its recently updated arXiv version \cite{jacobsen2024parameterfree} for the technical details. Our quantitative improvement over \cite{jacobsen2024parameterfree} will be the leading constant optimality.}}

To this end, we will take a detour through the \emph{continuous time} (CT), first solving a CT analogue of the problem, and then converting the solution back to \emph{discrete time} (DT). Quantitatively, our goal above can be seen as the gradient adaptive refinement of \cite{zhang2022pde} -- without considering gradient adaptivity, \cite{zhang2022pde} presents an algorithm designed in CT that \emph{natively} achieves the \textbf{optimal comparator adaptive bound} Eq.(\ref{eq:comparator_adaptivity}), while earlier algorithms designed in DT do not (unless they resort to the impractical \emph{doubling trick}) \cite{mcmahan2014unconstrained,orabona2016coin}. Broadly speaking, such a result exemplifies a higher level observation: while various benefits of the CT approach have been demonstrated in online learning before \cite{kapralov2011prediction,drenska2020prediction,kobzar2020a_new,zhang2022optimal,harvey2023optimal}, it remains unsatisfactory that no existing work (to the best of our knowledge) has used it to obtain DT gradient adaptive regret bounds, even though the CT analogue of gradient adaptivity is often natural\footnote{Example: in finance, the gradient variance is analogous to the \emph{price volatility}. This is ubiquitous in the continuous time modeling of financial instruments, such as the \emph{geometric Brownian motion}.} and fairly standard to achieve \cite{freund2009method,harvey2023optimal}. In other words, 

\begin{center}
\begin{minipage}{.9\textwidth}
\textbf{one would expect the CT approach to make gradient adaptivity (hence, simultaneous adaptivity) easier as well, but such a benefit has not been demonstrated in the literature.}
\end{minipage}
\end{center}

The key reason of this limitation appears to be the crudity of existing \emph{discretization arguments}, i.e., the modification applied to a CT algorithm and its analysis to make them work well in DT. The state-of-the-art technique, due to \cite{harvey2023optimal}, replaces the continuous derivative in \emph{potential-based} CT algorithms (i.e., FTRL \cite{abernethy2008competing} and \emph{dual averaging} \cite{nesterov2009primal}) by the discrete derivative, and consequently, the standard \emph{It\^{o}'s formula} in the CT regret analysis by the \emph{discrete It\^{o}'s formula}. Applying the discrete derivative amounts to implicitly assuming the worst case gradient magnitude ($\norm{g_t}=G$), therefore any gradient adaptivity in CT is lost in DT by construction. \textbf{The second goal of this paper, on the technical side, is to propose a refined discretization argument that preserves such gradient adaptivity.}

\subsection{Contribution}\label{subsection:contribution}

As motivated above, the contributions of this paper are twofold.

\paragraph{Quantitative side} Starting from the results of \cite{zhang2022pde,harvey2023optimal}, we first show that in a CT analogue of the OLO problem, simultaneous adaptivity is quite easy to obtain by combining It\^{o}'s formula and the \emph{Backward Heat Equation} (BHE), a partial differential equation characterizing the reversal of heat diffusion. Building on this observation, our main result is a new DT algorithm achieving the following \textbf{near-optimal simultaneously adaptive regret bound} without the impractical doubling trick. With an arbitrary hyperparameter $\eps>0$, in the asymptotic regime of large $\norm{u}$ and $V_T$, 
\begin{equation}\label{eq:contribution_big_oh}
\reg_T(Env,u)\leq \eps\cdot O\rpar{\sqrt{V_T}}+\norm{u}\cdot O\rpar{\sqrt{V_T\log(\norm{u}\eps^{-1})}\vee G\log(\norm{u}\eps^{-1})}.
\end{equation}
Furthermore, given any hyperparameter $\alpha>\half$, the multiplying constant on the leading order term $\norm{u}\sqrt{V_T\log (\norm{u}\eps^{-1})}$ is $\sqrt{4\alpha}$, almost\footnote{There is an additive factor on $V_T$ proportional to $(\alpha-\half)^{-1}G^2$. For any fixed $\alpha>\half$ this additive factor is negligible in the large-$V_T$ regime, but we cannot pick $\alpha=\half$.} matching the $\sqrt{2}$ lower bound from \cite{zhang2022pde}. Comparing it to the literature:
\begin{itemize}
\item Excluding the recent arXiv update \cite{jacobsen2024parameterfree}, this is the first simultaneously adaptive regret bound matching the optimal $O(\sqrt{V_T})$ rate (with respect to $V_T$ alone), improving a series of prior works \cite{cutkosky2018black,mhammedi2020lipschitz,chen2021impossible,jacobsen2022parameter}. However, we note that \cite[Section~7]{jacobsen2024parameterfree} presents a different tuning of the algorithm from \cite{jacobsen2022parameter}, achieving the same big-Oh bound as Eq.(\ref{eq:contribution_big_oh}).
\item Within simultaneously adaptive online learning, our algorithm nearly achieves the \emph{leading constant optimality}, improving all the aforementioned works (including \cite{jacobsen2024parameterfree} as well). 
\end{itemize}

We also generalize the above to the setting without a known Lipschitz constant $G$, making the algorithm and its regret bound \emph{scale-free}. Since the hyperparameter $\eps$ is truly ``unitless'' in our algorithm design, there is no need to estimate the \emph{time-uniform Lipschitz constant} $\max_{t\in[1:T]}\norm{g_t}$ at the beginning, which eliminates the \emph{range ratio problem} from existing works \cite{mhammedi2020lipschitz,jacobsen2022parameter}. This keeps the algebra simple, while also avoiding the standard range-ratio-induced penalties in the regret bound. 

\paragraph{Technical side} Our key technical innovation is a new gradient adaptive discretization argument, refining the non-adaptive one from \cite{harvey2023optimal}. The essential idea is connecting the CT algorithm and its DT analogue via a \emph{change of variables}, which allows using the \emph{exact} BHE from CT to simplify the complicated algebra in DT. For our specific problem of simultaneous adaptivity, this procedure is arguably easier and more intuitive than existing approaches \cite{cutkosky2018black,mhammedi2020lipschitz,jacobsen2022parameter} that tackle DT directly.

\subsection{Notation}\label{subsection:notation}

Let $C^{1,2}(\X)$ be the class of bivariate functions on an open set $\X$, continuously differentiable in their first argument and twice continuously differentiable in their second argument. For any $\Phi\in C^{1,2}(\X)$, let $\partial_1\Phi$ and $\partial_2\Phi$ be its first order partial derivatives with respect to the first and the second argument of $\Phi$. Similarly, $\partial_{11}\Phi$, $\partial_{12}\Phi$ and $\partial_{22}\Phi$ denote the second order partial derivatives ($\partial_{12}\Phi=\partial_{21}\Phi$). For all $x$ and $u$, define $\Phi^*_x(z)\defeq\sup_{y}[zy-\Phi(x,y)]$, i.e., the Fenchel conjugate of $\Phi$ with respect to its second argument. 

We define the \emph{imaginary error function} as $\erfi(x)=\int_0^x\exp(u^2)du$; this is scaled by $\sqrt{\pi}/2$ from the conventional definition, thus can also be queried from standard software packages like \textsc{SciPy} and \textsc{JAX}. Let $\erfi^{-1}$ be its inverse function. 

$\Pi_\X(x)$ is the Euclidean projection of $x$ onto a closed convex set $\X$. $\log$ represents natural logarithm when the base is omitted. Throughout this paper, $\norm{\cdot}$ denotes the Euclidean norm.

\section{Related work}\label{section:related}

\paragraph{Simultaneous adaptivity} Within the two types of adaptive online learning, comparator adaptivity is typically considered more challenging than gradient adaptivity. Therefore, existing works on simultaneous adaptivity are built on the core algorithmic frameworks of comparator adaptivity, with various ``gradient adaptive modifications''. Let us take a closer look.
\begin{itemize}
\item The state-of-the-art result for comparator adaptivity alone is due to \cite{zhang2022pde}, which achieves not only the order-optimal regret bound Eq.(\ref{eq:comparator_adaptivity}), but also the optimal leading constant $\sqrt{2}$. It relies on a CT analysis combining the \emph{potential verification} argument from \cite{harvey2023optimal} and the \emph{loss-regret duality} from \cite{mcmahan2014unconstrained}. Our algorithm is similar but also gradient adaptive. 

\item Simultaneous gradient and comparator adaptivity was achieved in \cite{cutkosky2018black} for the first time, through a nested \emph{coin betting} approach \cite{orabona2016coin}. Later on, \cite{mhammedi2020lipschitz} proposed potential-based algorithms that are also \emph{scale-free} \cite{orabona2018scale}, and the analysis utilizes a novel computer aided procedure. Both of them guarantee Eq.(\ref{eq:simultaneous_old}), which is $\sqrt{\log V_T}$ away from the optimal dependence on $V_T$. As discussed in Appendix~\ref{section:more_background}, the well-known doubling trick (i.e., restarting) is insufficient for closing this gap. 

\item Most recently, \cite{jacobsen2022parameter} achieved Eq.(\ref{eq:simultaneous_old}) (modulo $\log(\log)$ factors) using a regularized variant of \emph{Online Mirror Descent} (OMD). The framework is actually very general: its latest arXiv update \cite[Section~7]{jacobsen2024parameterfree} shows that under a different tuning, the same algorithm can achieve the $O\rpar{\norm{u}\sqrt{V_T\log \norm{u}}}$ bound we aim for, even without the $\log(\log)$ factor from \cite{jacobsen2022parameter}. Nonetheless, the leading constant in this bound is $6$, while we nearly improve it to the optimal value $\sqrt{2}$. 

\item Besides, if logarithmic terms are omitted altogether, then simultaneous adaptivity can also be achieved by a very powerful aggregation approach \cite{chen2021impossible}. Its practicality is limited by the increased computation: $O(T)$ base algorithms are maintained at the same time. 
\end{itemize}

\paragraph{Continuous time approach} Our result fits into an emerging direction of online learning: exploiting the synergy between CT and DT algorithms. Concrete benefits in DT, including better bounds and simpler analyses, have been demonstrated in various settings of minimax online learning \cite{bayraktar2020finite,drenska2020prediction,kobzar2020a_new,wang2022new,harvey2023optimal} and adaptive online learning \cite{kapralov2011prediction,daniely2019competitive,portella2022continuous,zhang2022pde,zhang2022optimal}. Conversely, such a synergy can benefit traditional ``model-based'' CT decision making as well; for example, the celebrated \emph{Black-Scholes model} \cite{black1973pricing} for option pricing can be derived as the scaling limit of a DT adversarial online learning model, which provides a strong justification of its validity \cite{abernethy2012minimax,abernethy2013hedge}. 

Most of these works establish the CT-DT synergy via potential-based algorithms. Roughly speaking,\footnote{For clarity, we do not account for gradient adaptivity in this discussion. Essentially, the idea of second order gradient adaptivity is replacing $t$ by the running gradient variance $V_t=\sum_{i=1}^t\norm{g_i}^2$.} the decision at time $t$ has the form $\phi'_t(S_t)$, which denotes the derivative of a potential function $\phi_t$ evaluated at a ``sufficient statistic'' $S_t$ that summarizes the history. In the CT regime, the crucial simplicity is that a suitable $\phi_t$ satisfies a PDE, thus finding it can be a tractable task. The tricky step is to properly convert this CT algorithm to DT and quantify their performance discrepancy, which we call the discretization argument. 

The most natural idea is to apply the CT algorithm to DT as is, and characterize the performance discrepancy using Taylor's theorem \cite{abernethy2013hedge,kobzar2020a_new}. However, this approach requires a terminal condition at a fixed time horizon $T$, which is missing from many settings of adaptive online learning. Harvey et al. \cite{harvey2023optimal} proposed a particularly strong and elegant alternative: replacing the standard derivative $\phi'_t(S_t)$ in the CT algorithm by the discrete derivative, e.g., $\frac{1}{2G}[\phi_t(S_t+G)-\phi_t(S_t-G)]$. Then, the performance discrepancy between CT and DT can be characterized by a DT analogue of the PDE, which can be analyzed in a principled manner. Such an analysis has been adopted in several recent works \cite{greenstreet2022efficient,portella2022continuous,zhang2022pde,zhang2022optimal}, but the downside is that any gradient adaptive upgrade on the CT algorithm (not hard to obtain, as shown in Section~\ref{section:continuous}) is lost in DT by construction. The present work addresses this limitation.

\section{Warm up: Adaptivity in continuous time}\label{section:continuous}

To begin with, we study a one dimensional continuous time analogue of the unconstrained OLO problem, in order to demonstrate the inherent simplicity of simultaneous adaptivity. The restriction to 1D is justified by a well-known polar-decomposition technique from \cite{cutkosky2018black}, which will be made concrete in Section~\ref{section:discretization}. 

Technically, much of this section is standard: the critical use of It\^{o}'s formula in CT online learning was pioneered by Freund \cite{freund2009method} and greatly streamlined by Harvey et al. \cite[Appendix~B]{harvey2023optimal}. Our treatment of simultaneous adaptivity will follow Harvey et al.'s argument and a classical loss-regret duality from \cite{mcmahan2014unconstrained}. Nonetheless, it embodies the key intuition, thus paves the way for our main contribution on the discretization argument.

\subsection{Setting}

Unlike the DT adversarial setting universally recognized as a repeated game, the definition of a reasonable CT analogue has been elusive. First, following \cite{harvey2023optimal}, we model the combined actions of the CT environment, i.e., the CT analogue of the gradient sum $\sum_{i=1}^tg_i$, as an arbitrary \emph{continuous semimartingale} denoted by $S_t$ ($t\in\R_{\geq 0}$),\footnote{In DT, the gradient sum $\sum_{i=1}^tg_i$ is a classical quantity for dual space online learning algorithms, sometimes called the ``sufficient statistic''. In the CT analogue, we essentially assume the sufficient statistic evolves as a stochastic process with a very general law. Our treatment has a slight difference from \cite{harvey2023optimal}: the latter models $\abs{S_t}$ rather than $S_t$.} with $S_0=0$ -- examples include the \emph{Brownian motion} and the \emph{It\^{o} process}.\footnote{It\^{o}'s process is a general form of diffusion defined by a differential equation $\d S_t=\sigma_t\d B_t+\mu_t\d t$, where $B$ is the standard Brownian motion. Here, the \emph{diffusion coefficient} $\sigma$ and the \emph{drift coefficient} $\mu$ are both stochastic processes. Rigorous definitions of relevant stochastic process concepts can be found in \cite{revuz2013continuous}.} Such an assumption is motivated by the analysis: semimartingales form the largest class of integrators with respect to which the It\^{o} integral can be defined; ``continuous'' means that the sample paths are continuous almost surely -- this avoids the jump-correction terms in the It\^{o}'s formula. 

Next, for any continuous semimartingale $S$, one can define its \emph{quadratic variation process}, denoted by $[S]_t$. Intuitively, $[S]_t$ is the CT analogue of the 1D gradient variance $V_t=\sum_{i=1}^tg_i^2$. For the standard Brownian motion, $[S]_t=t$. For the It\^{o}'s process, $[S]_t=\int_0^t\sigma^2_s\d s$, where $\sigma$ is the diffusion coefficient of $S$ (note that $\sigma$ is itself a stochastic process). 

After characterizing the CT environment, let us turn to the CT learner. We consider the potential framework. The learner fixes a potential function $\phi\in C^{1,2}(\X)$ at the beginning, where $\X$ is an open set containing $\R_{\geq 0}\times\R$. The learner's decision against the environment $S$ is $\partial_2\phi([S]_t,-S_t)$, which is a continuous process -- this mirrors the standard FTRL family in DT \cite[Chapter~7]{orabona2023modern}. Furthermore, the ``adversarialness'' of the DT setting is analogous to the fact that the law of the CT environment $S$ can depend on the learner's potential function $\phi$. 

With the above, the learner's total loss over a time horizon $T$ is the It\^{o} integral $\int_0^T\partial_2\phi([S]_t,-S_t)\d S_t$ \cite[Definition~IV.2.9]{revuz2013continuous}, and any fixed comparator $u$ induces the total loss $uS_T$. The goal of this CT problem is thus choosing $\phi$ to minimize the \emph{continuous time regret},
\begin{equation*}
\reg^{\mathrm{CT}}_T(Env,u)\defeq\rpar{\int_0^T\partial_2\phi([S]_t,-S_t)\d S_t}-uS_T.
\end{equation*}

\subsection{Analysis}

The crucial simplicity of the CT setting can be seen in the following theorem. This is new to the literature, but just a combination of steps in \cite[Theorem~B.2]{harvey2023optimal} and \cite[Theorem~9.6]{orabona2023modern}.

\begin{theorem}\label{theorem:continuous}
If $\phi\in C^{1,2}(\X)$ satisfies the Backward Heat Equation (BHE) $\partial_1\phi+\half\partial_{22}\phi=0$, then for all $T\in\R_{\geq 0}$ and $u\in\R$, almost surely,
\begin{equation*}
\reg^{\mathrm{CT}}_T(Env,u)\leq \phi(0,0)+\phi^*_{[S]_T}(u).
\end{equation*}
Here we follow the notation from Section~\ref{subsection:notation}: $\phi^*_{\cdot}(\cdot)$ is the Fenchel conjugate of $\phi$ with respect to its second argument.
\end{theorem}

Let us include the proof for completeness, which also highlights the ideal type of analysis that the DT regime should also follow. The central component is the It\^{o}'s formula, i.e., the stochastic analogue of the chain rule. The specific version below combines two results from \cite{revuz2013continuous}: Proposition~IV.1.18, and Remark 1 after Theorem~IV.3.3. 

\begin{lemma}[It\^{o}'s formula]\label{lemma:ito}
If $f\in C^{1,2}(\X)$ and $X$ is a continuous semimartingale, then for all $T\in\R_{\geq 0}$, almost surely,
\begin{equation*}
f\rpar{[X]_T, X_T}-f\rpar{0,X_0}=\int_0^T\partial_2f\rpar{[X]_t, X_t}\d X_t+\int_0^T\spar{\partial_1f\rpar{[X]_t, X_t}+\frac{1}{2}\partial_{22}f\rpar{[X]_t, X_t}}\d [X]_t.
\end{equation*}
\end{lemma}

In the language of the potential framework, the It\^{o}'s formula is a potential verification argument (on $f$), but in the strongest form: an equality. The regret bound can then be proved with ease. 

\begin{proof}[Proof of Theorem~\ref{theorem:continuous}]
Applying Lemma~\ref{lemma:ito} with $f\leftarrow -\phi$ and $X\leftarrow -S$, and further using the BHE $\partial_1\phi+\half\partial_{22}\phi=0$ to eliminate the integral with respect to $[S]$, we have
\begin{align*}
\reg^{\mathrm{CT}}_T(Env,u)&=\phi\rpar{0,0}-uS_T-\phi\rpar{[S]_T, -S_T}\\
&\leq \phi\rpar{0,0}+\sup_{y\in\R}\spar{uy-\phi\rpar{[S]_T, y}}.
\end{align*}
The proof is complete by plugging in the definition of the Fenchel conjugate $\phi^*_{[S]_T}(u)$.
\end{proof}

It remains to pick a specific potential function $\phi$. Similar to \cite{zhang2022pde,harvey2023optimal}, with arbitrary constants $\eps,\delta>0$, we define
\begin{equation}\label{eq:ct_potential}
\phi^{\mathrm{CT}}(x,y)=\eps\sqrt{x+\delta}\rpar{2\int_0^{\frac{y}{\sqrt{2(x+\delta)}}}\erfi(u)du-1},
\end{equation}
which satisfies the BHE. Also, the shift $\delta$ on the first argument $x$ ensures $\phi^{\mathrm{CT}}\in C^{1,2}(\R_{>-\delta}\times\R)$. Plugging in the Fenchel conjugate computation from \cite[Theorem~4]{zhang2022pde}, Theorem~\ref{theorem:continuous} becomes
\begin{equation}\label{eq:ct_regret}
\reg^{\mathrm{CT}}_T(Env,u)\leq \eps\sqrt{[S]_T+\delta}+\abs{u}\sqrt{2\rpar{[S]_T+\delta}}\spar{\sqrt{\log\rpar{1+\frac{\abs{u}}{\sqrt{2}\eps}}}+1}.
\end{equation}
With $\eps=1$, $\reg^{\mathrm{CT}}_T(Env,u)=O\rpar{\abs{u}\sqrt{[S]_T\log\abs{u}}}$, which is the desirable CT simultaneously adaptive bound, analogous to the $O\rpar{\abs{u}\sqrt{V_T\log\abs{u}}}$ bound we aim for in DT. Furthermore, the leading constant $\sqrt{2}$ in Eq.(\ref{eq:ct_regret}) matches the optimal leading constant in the DT setting \cite{zhang2022pde}.

To conclude, the key takeaway is that in CT, one can use the It\^{o}'s formula and the Backward Heat Equation to achieve an ``ideal form'' of simultaneous adaptivity fairly easily. In some sense, they capture the important problem structure, which suggests that their DT analogues could potentially improve and simplify prior works that do not exploit such structures. Next, we make this intuition concrete.

\section{Main result: Refined discretization}\label{section:discretization}

In this section, we consider a DT setting that slightly generalizes the one at the beginning of this paper. Let us assume the Lipschitz constant $G$ is unknown, but at the beginning of each (the $t$-th) round we have access to a \emph{hint} $h_t$ which satisfies $h_{t}\geq h_{t-1}$ and $\norm{g_t}\leq h_t$ (initially, assume $h_0=h_1>0$ w.l.o.g.). Such a setting is motivated by \cite{cutkosky2019artificial}, where designing a full Lipschitzness-adaptive algorithm can be reduced to solving this OLO problem with hints; details will be discussed in Section~\ref{section:lipschitz}. For clarity, one may think of $h_t=G$ when $G$ is known.

Our solution centers around a 1D potential function $\Phi_t$ defined by a \emph{change of variables}. With hyperparameters $\eps,\alpha,k_t>0$ and $z_t>k_th_t$, define
\begin{equation}\label{eq:potential}
\Phi_t(V,S)\defeq\phi(V+z_t+k_tS,S),
\end{equation}
where
\begin{equation*}
\phi(x,y)\defeq\eps\sqrt{\alpha x}\rpar{2\int_0^{\frac{y}{\sqrt{4\alpha x}}}\erfi(u)du-1}.
\end{equation*}
Here, $\phi$ is a generalization of the CT potential (\ref{eq:ct_potential}), satisfying $\partial_1\phi+\alpha\partial_{22}\phi=0$, the generalized Backward Heat Equation with constant $\alpha$. It can be verified that $\phi\in C^{1,2}(\R_{>0}\times\R)$. 

Intuitively, $\Phi_t$ is the potential function we apply in the $t$-th round, therefore $k_t$ and $z_t$ should be functions of the hint $h_t$. By a simple dimensional analysis, $z_t\propto h_t^2$, $k_t\propto h_t$, while $\eps$ and $\alpha$ are real numbers. Also, the definition of $\Phi_t(V,S)$ is only valid when $S$ is larger than a threshold, since the first argument of $\phi$ can only be positive -- the choice of $z_t$ and $k_t$ will ensure that all the ``interesting'' values of $S$ are above this threshold.

\subsection{Algorithm}

Similar to many other comparator adaptive OLO algorithms, our algorithm has a two-level hierarchical structure. On the high level is the meta algorithm from \cite[Section~3]{cutkosky2018black}, decomposing the OLO problem on $\R^d$ into two independent subtasks: learning the direction and the magnitude of the comparator. Direction learning is handled by gradient adaptive OGD on a unit norm ball. Magnitude learning is handled by the novel 1D base algorithm employing the potential function $\Phi_t$, followed by a constraint-imposing technique \cite[Section~4]{cutkosky2020parameter} which restricts its output from $\R$ to $\R_{\geq 0}$. 

Concretely, we present the important 1D base algorithm as Algorithm~\ref{alg:base}, while the meta algorithm is presented as Algorithm~\ref{alg:meta}. Since the base algorithm is updated using the surrogate feedback provided by the meta algorithm, we denote these surrogate algorithmic quantities with tilde. 

\begin{algorithm*}[ht]
\caption{1D base algorithm\label{alg:base}}
\begin{algorithmic}[1]
\REQUIRE The potential function $\Phi_t$ defined in Eq.(\ref{eq:potential}). Hints $h_1,h_2,\ldots\in\R_{> 0}$ satisfying $h_t\geq h_{t-1}$. Surrogate loss gradients $\tilde l_1,\tilde l_2,\ldots$ satisfying $\tilde l_t\in[-h_t,h_t]$ and $\sum_{i=1}^t\tilde l_i\leq h_t$ for all $t$.
\STATE Initialize $\tilde V_0=0$, $\tilde S_0=0$.
\FOR{$t=1,2,\ldots$}
\STATE Receive the hint $h_t$ and use it to define $k_t$ and $z_t$ in the potential function $\Phi_t$.
\STATE Predict $\tilde y_t=\partial_2\Phi_t(\tilde V_{t-1},\tilde S_{t-1})$.
\STATE Receive the surrogate loss gradient $\tilde l_t$.
\STATE Let $\tilde V_t=\tilde V_{t-1}+\tilde l^2_t$, and $\tilde S_t=\tilde S_{t-1}-\tilde l_t$.
\ENDFOR
\end{algorithmic}
\end{algorithm*}

\begin{algorithm*}[ht]
\caption{Meta algorithm on $\R^d$.\label{alg:meta}}
\begin{algorithmic}[1]
\STATE Define $\A_{1d}$ as a copy of Algorithm~\ref{alg:base}. Define $\A_\ball$ as OGD on the $d$-dimensional unit $L_2$ norm ball, with adaptive learning rate $\eta_t=\sqrt{2/\sum_{i=1}^t\norm{g_i}^2}$. The initialization of $\A_\ball$ is arbitrary. 
\FOR{$t=1,2,\ldots$}
\STATE Query $\A_{1d}$ for its prediction $\tilde y_t\in\R$. Let $y_t=\Pi_{\R_+}(\tilde y_t)$.
\STATE Query $\A_\ball$ for its prediction $w_t\in\R^d$; $\norm{w_t}\leq 1$. 
\STATE Predict $x_t=w_ty_t$, receive the loss gradient $g_t\in\R^d$.
\STATE Send $g_t$ as the surrogate loss gradient to $\A_\ball$.
\STATE Define $l_t=\inner{g_t}{w_t}$, and
\begin{equation*}
\tilde l_t=\begin{cases}
l_t,& l_t\tilde y_t\geq l_ty_t,\\
0,& \textrm{else}.
\end{cases}
\end{equation*}
\STATE Send $\tilde l_t$ as the surrogate loss gradient to $\A_{1d}$.
\ENDFOR
\end{algorithmic}
\end{algorithm*}

Notice that Algorithm~\ref{alg:base} requires a somewhat nonstandard condition, $\sum_{i=1}^t\tilde l_i\leq h_t$ for all $t$. This is to ensure that the update $\tilde y_t=\partial_2\Phi_t(\tilde V_{t-1},\tilde S_{t-1})$ is well-defined: $\tilde S_{t-1}=-\sum_{i=1}^t\tilde l_i\geq -h_t$, therefore with $z_t>k_th_t$, we always have $\tilde V_{t-1}+z_t+k_t\tilde S_{t-1}>0$, which complies with the positivity requirement on the first argument of $\phi$. The following lemma, proved in Appendix~\ref{section:appendix_meta}, shows that the surrogate losses defined by the meta algorithm indeed satisfy this requirement, thus the entire algorithm procedure is well-posed.

\begin{restatable}[Well-posedness]{lemma}{well}\label{lemma:well_pose}
The surrogate loss $\tilde l_t$ defined in Algorithm~\ref{alg:meta} satisfies $\sum_{i=1}^t\tilde l_i\leq h_t$ for all $t$.
\end{restatable}

\subsection{Analysis}

Turning to the analysis, our key innovation is the following lemma. Besides it, the rest of the analysis is somewhat standard. 

\begin{restatable}[Key lemma: one step potential bound]{lemma}{onestep}\label{lemma:one_step}
Let $\eps>0$, $\alpha>\half$, and for all $t$, $k_t=2h_t$ and $z_t=\frac{12\alpha +4}{2\alpha-1}h_t^2$. Then, the 1D potential functions satisfy
\begin{equation*}
\Phi_t(V+c^2,S+c)-\Phi_{t-1}(V,S)-c\partial_2\Phi_t(V,S)\leq 0,
\end{equation*}
for all $V\geq 0$, $S\geq -h_{t-1}$ and $c\in[-h_t-\min(S,0),h_t]$.
\end{restatable}

This is a potential verification argument, serving a similar purpose as the It\^{o}'s formula in the CT analysis (Lemma~\ref{lemma:ito}). The condition on $c$ simply means we require $c\in[-h_t,h_t]$ and $S+c\geq -h_t$. With the lemma, one can take a telescopic sum with $c=-\tilde l_t$, which returns a cumulative loss upper bound of the 1D base algorithm (Algorithm~\ref{alg:base}): $\sum_{t=1}^T\tilde l_t\tilde y_t\leq \Phi_0\rpar{0,0}-\Phi_T\rpar{\tilde V_{T},\tilde S_{T}}$. Then, similar to what we did in CT, the regret bound of Algorithm~\ref{alg:base} follows from the loss-regret duality \cite{mcmahan2014unconstrained} and a Fenchel conjugate computation (Lemma~\ref{lemma:conjugate}).

Now let us sketch the proof of Lemma~\ref{lemma:one_step}. 

\paragraph{Proof sketch of Lemma~\ref{lemma:one_step}} The proof is structured into three steps.
\begin{enumerate}
\item Proving $\Phi_{t}(V,S)\leq\Phi_{t-1}(V,S)$. 

This is due to $\partial_1\phi(x,y)\leq 0$. After that, we have
\begin{equation*}
\Phi_t(V+c^2,S+c)-\Phi_{t-1}(V,S)-c\partial_2\Phi_t(V,S)\leq \Phi_t(V+c^2,S+c)-\Phi_{t}(V,S)-c\partial_2\Phi_t(V,S),
\end{equation*}
and it suffices to show $\rhs\leq 0$. Since all the subscripts are $t$ now, let us simply drop this subscript from $\Phi$, $z$ and $k$, and write $G$ in place of $h_t$. 

\item Convert checking $\Phi$ to checking $\phi$, using the change of variables. 

Let us define
\begin{equation*}
f_{V,S}(c)=\Phi(V+c^2,S+c)-\Phi(V,S)-c\partial_2\Phi(V,S).
\end{equation*}
The task now is showing $f_{V,S}(c)\leq 0$. Taking the derivatives with respect to $c$,
\begin{equation*}
f'_{V,S}(c)=2c\partial_1\Phi(V+c^2,S+c)+\partial_2\Phi(V+c^2,S+c)-\partial_2\Phi(V,S),
\end{equation*}
\begin{align*}
f''_{V,S}(c)&=2\partial_1\Phi+4c^2\partial_{11}\Phi+4c\partial_{12}\Phi+\partial_{22}\Phi\Big|_{(V+c^2,S+c)}\\
&\leq 2\partial_1\Phi+4G^2\partial_{11}\Phi+4G\abs{\partial_{12}\Phi}+\partial_{22}\Phi\Big|_{(V+c^2,S+c)},
\end{align*}
where the final subscript means all the involved partial derivatives are evaluated at $(V+c^2,S+c)$. Notice that $f_{V,S}(0)=f'_{V,S}(0)=0$. Therefore, to prove $f_{V,S}(c)\leq 0$, it suffices to show $f''_{V,S}(c)\leq 0$ for all considered values of $V$, $S$ and $c$. 

Crucially, due to the change of variables, partial derivatives of $\Phi$ can be easily rewritten using partial derivatives of $\phi$ (Appendix~\ref{section:potential_fact}). Plugging that in, it suffices to verify the following two cases. 

\textbf{Case 1.}\quad If $\partial_{12}\Phi(V+c^2,S+c)\leq 0$, then
\begin{equation*}
2\partial_1\phi+(k-2G)^2\partial_{11}\phi+2(k-2G)\partial_{12}\phi+\partial_{22}\phi\Big|_{(V+c^2+z+k(S+c),S+c)}\leq 0.
\end{equation*}
\textbf{Case 2.}\quad If $\partial_{12}\Phi(V+c^2,S+c)> 0$, then
\begin{equation*}
2\partial_1\phi+(k+2G)^2\partial_{11}\phi+2(k+2G)\partial_{12}\phi+\partial_{22}\phi\Big|_{(V+c^2+z+k(S+c),S+c)}\leq 0.
\end{equation*}

\item Controlling $\partial_{11}\phi$ and $\partial_{12}\phi$ by picking $k$ and $z$, and applying the BHE. 

Closely examining the above two inequalities on $\phi$, one could notice a striking similarity with the BHE $\partial_1\phi+\half\partial_{22}\phi=0$, which the ideal CT potential $\phi^{\mathrm{CT}}$ was \emph{designed} to satisfy, cf., Eq.(\ref{eq:ct_potential}). In particular, if one could drop the two annoying terms, $\partial_{11}\phi$ and $\partial_{12}\phi$, then $\phi^{\mathrm{CT}}$ already fits into the above two cases with \emph{equality}. Essentially, the exhibited similarity is due to the fact that both $\phi$ (in DT) and $\phi^{\mathrm{CT}}$ (in CT) have their ``time variable'' (i.e., their first argument) growing according to the quadratic variation of the environment. Therefore, it appears to be an important problem structure, rather than a coincidence that happens to work in our favor.

With that, our idea is to pick $k$ and $z$ such that the annoying residual terms are upper bounded by a small constant multiplying $\partial_{22}\phi$. Then, we can still use the BHE $\partial_1\phi+\alpha\partial_{22}\phi=0$, but with a different diffusivity constant $\alpha>\half$, to control the LHS of the above two cases. Eventually, this will only cost us a slightly suboptimal leading constant in the regret bound ($\sqrt{4\cdot\half}=\sqrt{2}\Rightarrow\sqrt{4\alpha}$).

Formally, to this end, notice that $k=2G$ trivially satisfies Case 1. Case 2 is a bit more involved, but from the full expressions of $\partial_{11}\phi$, $\partial_{12}\phi$ and $\partial_{22}\phi$, it is not hard to see that a large enough $z\geq \frac{12\alpha +4}{2\alpha-1}G^2$ suffices. This completes the proof. 
\end{enumerate}

With Lemma~\ref{lemma:one_step} above, it is fairly straightforward to obtain the regret bound of the 1D base algorithm (Lemma~\ref{lemma:base}), as we sketched earlier. Then, since the meta algorithm is simply the combination of two existing black-box reductions, its regret bound follows from \cite[Theorem~2]{cutkosky2018black} and \cite[Theorem~2]{cutkosky2020parameter}. This returns the following theorem as the main result of this paper. 

\begin{restatable}[Main result]{theorem}{meta}\label{thm:meta}
With $\eps>0$, $\alpha>\half$, $k_t=2h_t$ and $z_t=\frac{12\alpha +4}{2\alpha-1}h_t^2$, Algorithm~\ref{alg:meta} guarantees for all $T\in\N_+$ and $u\in\R^d$, 
\begin{equation*}
\reg_T(Env,u)\leq \eps\sqrt{\alpha\rpar{V_T+z_T+k_T\bar S}}+\norm{u}\rpar{\bar S+2\sqrt{2V_T}},
\end{equation*}
where
\begin{equation*}
\bar S= 4\alpha k_T\rpar{1+\sqrt{\log(2\norm{u}\eps^{-1}+1)}}^2+\sqrt{4\alpha \rpar{V_T+z_T}}\rpar{1+\sqrt{\log(2\norm{u}\eps^{-1}+1)}}.
\end{equation*}
\end{restatable}

Theorem~\ref{thm:meta} contains the precise regret bound without any big-Oh. Nonetheless, one could use asymptotic orders to make it more interpretable (see Appendix~\ref{section:appendix_meta} for the derivation).
\begin{equation*}
\reg_T(Env,u)\leq \eps\rpar{\sqrt{\alpha\rpar{V_T+\frac{12\alpha +4}{2\alpha-1}h_T^2}}+4\alpha h_T}+\norm{u} O\rpar{\sqrt{V_T\log(\norm{u}\eps^{-1})}\vee h_T\log(\norm{u}\eps^{-1})},
\end{equation*}
which is simultaneously valid in two regimes: ($i$) $\norm{u}\gg\eps$ and $V_T\gg h_T^2$; and ($ii$) $u=0$. In comparator adaptive online learning, regret bounds of this form are said to characterize the \emph{loss-regret tradeoff} \citep{zhang2022pde}: with a small $\eps$, one could ensure that the cumulative loss $\reg_T(Env,0)$ is low, while only sacrificing a $\sqrt{\log(\eps^{-1})}$ penalty on the leading term of the regret bound. We also note that the additive logarithmic ``residual'' term $h_T\log(\norm{u}\eps^{-1})$ is standard in simultaneously adaptive regret bounds, and not removable in some sense \cite[Section~5.5.1]{cutkosky2018algorithms}. 

The key strength of this bound is that, the dependence on $V_T$ alone is $O(\sqrt{V_T})$, matching the optimal gradient adaptive bound achieved by OGD. Furthermore, in the regime of large $\norm{u}$ and $V_T$, the leading order term is $\sqrt{4\alpha V_T\log(\norm{u}\eps^{-1})}$, where the multiplying constant $\sqrt{4\alpha}$ almost matches the $\sqrt{2}$ lower bound from \citep{zhang2022pde}. 

\section{Extension: Unknown Lipschitz constant without hints}\label{section:lipschitz}

Finally, we discuss the full extension of our main result to the setting with unknown $G$, removing the need of hints. This follows from a reduction to our Section~\ref{section:discretization}, developed by \citep{cutkosky2019artificial,mhammedi2020lipschitz}. Let us define $G_t\defeq\max_{i\leq t}\norm{g_t}$, and w.l.o.g., $G\defeq G_T$. 

The essential idea is the following. Without knowing $G$, we use the hint $h_t$ as a guess of the ``running Lipschitz constant'' $G_t$, before observing $\norm{g_t}$. Naturally, $h_t=G_{t-1}$ makes a reasonable guess, but there is always some chance of ``surprise'', where $\norm{g_t}>G_{t-1}$, and the analysis from Section~\ref{section:discretization} breaks. To fix this issue, instead of feeding the algorithm the true gradient $g_t$, one could feed its \emph{clipped version}, $\bar g_t=g_tG_{t-1}/G_t$. Now, $h_t=G_{t-1}$ is always a valid hint for the clipped gradient $\bar g_t$, therefore our main result (Theorem~\ref{thm:meta}) can be applied in a black-box manner. 

Ultimately, we care about the regret evaluated on the true gradients $g_{1:T}$, rather than the clipped gradients $\bar g_{1:T}$. Their difference is related to the magnitude of the predictions $\norm{x_t}$, thus one could use the standard constraint-imposing technique \cite[Theorem~2]{cutkosky2020parameter} once again to control it. In combination, this yields the following lemma \citep[Corollary~3]{mhammedi2020lipschitz}.
\begin{lemma}\label{lemma:scale_free}
If we denote the simultaneously adaptive regret bound in our Theorem~\ref{thm:meta} as a function $R(u,V_T,h_T)$, then there exists an algorithm taking ours as a black-box subroutine, guaranteeing
\begin{equation*}
\reg_T(Env,u)\leq R(u,V_T,G)+G\norm{u}^3+G\sqrt{\max_{t\leq T}B_t}+G\norm{u},
\end{equation*}
where $B_t\defeq\sum_{i=1}^t\norm{g_i}/G_t$. 
\end{lemma}
Notably, this algorithm needs neither the knowledge of $G$, nor any other oracle knowledge like hints. Moreover, no restarting (e.g., the doubling trick) is required. In the special case with bounded domain ($\norm{x_t},\norm{u}\leq D$), a simplified variant of this procedure guarantees an even smaller bound, $\reg_T(Env,u)\leq R(u,V_T,G)+2DG$.

The strength of our main result can be demonstrated in this general setting as well. 
\begin{itemize}
\item First, the algorithm obtained from Section~\ref{section:discretization} and Lemma~\ref{lemma:scale_free} is \emph{scale-free} \citep{orabona2018scale}, in the sense that if all the loss gradients are scaled by $c\in\R_{>0}$, then the prediction $x_t$ of the algorithm remains unchanged, and the above $G$-adaptive regret bound is scaled by exactly $c$. This is a favorable property in practice, also satisfied by prior works \citep{mhammedi2020lipschitz,jacobsen2022parameter} which we quantitatively improve.
\item Second, we avoid the \emph{range ratio} problem in prior works \citep{mhammedi2020lipschitz,jacobsen2022parameter}. For OLO with hints, existing analogues of our Theorem~\ref{thm:meta} have the shape like $\reg_T(Env,u)\leq O\rpar{\norm{u}\sqrt{V_T\log (\norm{u} V_T/h^2_1)}}$, where the range ratio $V_T/h_1^2$ can be arbitrarily large despite being inside the log. Using a restarting trick \citep{mhammedi2020lipschitz} or a somewhat more complicated ``soft-thresholding'' on the prediction $x_t$ \citep{jacobsen2022parameter}, one could replace this range ratio by $\mathrm{poly}(T)$. However, this sacrifices either the practicality of the algorithm or its analytical simplicity, and in both cases, one is left with a suboptimal $\sqrt{\log T}$ multiplicative factor on the leading term of the regret bound.\footnote{We note that the updated \cite[Section~7]{jacobsen2024parameterfree} presents a different tuning of the algorithm from \cite{jacobsen2022parameter}, which avoids the range ratio problem as well. In some sense, the original tuning of the algorithm is the OMD analogue of \cite{mhammedi2020lipschitz}, whereas the new tuning from \cite{jacobsen2024parameterfree} is the OMD analogue of the $\erfi$ potential algorithm.} 

Essentially, such a range ratio problem originated in \citep{mhammedi2020lipschitz} due to the existence of ``unit'' in their confidence hyperparameter $\eps$. Analogous to \citep{mcmahan2014unconstrained,orabona2016coin}, the paper applies the potential function
\begin{equation*}
\Phi_t(V,S)=\frac{\eps}{\sqrt{V}}\exp\rpar{\frac{S^2}{2V+2h_t\abs{S}}},
\end{equation*}
where $\eps$ has the unit of $G^2$ due to a dimensional analysis. If $G$ is known, then one could pick $\eps\propto G^2$, leading to the regret bound Eq.(\ref{eq:simultaneous_old}). Without knowing $G$, since $\eps$ needs to be determined at the beginning, the feasible choice becomes $\eps\propto h_1^2$ -- this replaces $G^2$ in Eq.(\ref{eq:simultaneous_old}) by $h^2_1$, causing the range ratio problem. 

Our algorithm has an important difference. The confidence hyperparameter $\eps$ in our potential function Eq.(\ref{eq:potential}) is ``unitless'', therefore when selecting it at the beginning, we do not need a guess of $G$. This avoids the range ratio problem, since there are no $V_T/h_1^2$ or $G/h_1$ terms in our regret bound at all. Neither restarting nor soft-thresholding is needed. The takeaway is that, such a range ratio problem appears to be an analytical artifact due to certain unit inconsistency (and ultimately, the suboptimal loss-regret tradeoff \citep{zhang2022pde}), which can be eliminated by a better design of the potential function. 
\end{itemize}

\section{Conclusion}

The present work studies how to achieve simultaneous gradient and comparator adaptivity in OLO with Lipschitz losses. A new continuous-time-inspired algorithm is proposed, achieving the state-of-the-art regret bound. The crucial technique is a new discretization argument that preserves gradient adaptivity from CT to DT, improving an already powerful, but non-gradient-adaptive one from the literature. This could be of broader applicability, and a natural step forward is exploiting the benefits of this technique in other online learning problems of interest. Finally, the extension to the setting with unknown Lipschitz constant is discussed, where our algorithm is made scale-free.

\section*{Acknowledgement}

We thank Andrew Jacobsen for the very helpful arXiv update \cite{jacobsen2024parameterfree}, and the anonymous reviewers for their constructive feedback. This research was partially supported by the NSF under grants CCF-2200052, DMS-1664644, and IIS-1914792, by the ONR under grant N00014-19-1-2571, by the DOE under grant DE-AC02-05CH11231, by the NIH under grant UL54 TR004130, and by Boston University and Harvard University.

\bibliography{Discretized}

\newpage
\section*{Appendix}
\appendix

Appendix~\ref{section:more_background} discusses the limitation of the doubling trick for our problem. Appendix~\ref{section:potential_fact} contains the basic facts of the $\erfi$ potential. Appendix~\ref{section:base_proof} and \ref{section:appendix_meta} prove our main results. 

\section{Additional background}\label{section:more_background}

One of the main motivations of this work is to improve simultaneously adaptive regret bounds in the form of $O\rpar{\norm{u}\sqrt{V_T\log (G^{-2}\norm{u} V_T)}}$, i.e., Eq.(\ref{eq:simultaneous_old}), to $O\rpar{\norm{u}\sqrt{V_T\log \norm{u}}}$. Readers familiar with prior works (e.g., \cite{mcmahan2014unconstrained,zhang2022pde}) may think of it as a ``tuning issue'' solvable by the well-known \emph{doubling trick} (e.g., \cite[Section~2.3.1]{shalev2011online}) -- restarting existing algorithms \cite{cutkosky2018black,mhammedi2020lipschitz,jacobsen2022parameter} with a doubling ``confidence hyperparameter'' whenever the observed gradient variance exceeds a doubling threshold. We now discuss this idea further. 

\paragraph{Existing bound $+$ doubling trick} Using $\hat O$ to hide $\log(\log)$ factors, \citep[Theorem~1]{jacobsen2022parameter} achieves
\begin{equation*}
\mathrm{Regret}_T(u)\leq\hat O\rpar{\eps G+\norm{u}\spar{\sqrt{V_T\log\rpar{\frac{\norm{u}\sqrt{V_T}}{\eps G}+1}}\vee G\log\rpar{\frac{\norm{u}\sqrt{V_T}}{\eps G}+1}}}.
\end{equation*}
With a known $V_T$ budget, setting $\eps=O(\sqrt{V_T}/G)$ yields the desirable bound $\hat O(\norm{u}\sqrt{V_T\log\norm{u}})$, up to a ``morally secondary term'' $\norm{u}G\log\norm{u}$. Applying the standard doubling trick can relax the known-$V_T$ assumption. The only small price to pay is that the secondary term is multiplied by $O(\log T)$. 

Apart from that, \citep{cutkosky2018black,mhammedi2020lipschitz} cannot be applied similarly with the doubling trick, due to their ``inappropriately scaled'' logarithmic factors. Even when $V_T$ is known, the hyperparameter $\eps$ there cannot be set accordingly to achieve the $O(\norm{u}\sqrt{V_T\log\norm{u}})$ regret bound we aim for. 

\paragraph{Weakness of doubling trick} A major limitation of the doubling trick is its practicality. Concretely, \citep[Section~5]{besson2018doubling} evaluates several versions of the doubling trick in the bandit setting. In almost all cases, the combination of the fixed-$T$ algorithm and the doubling trick performs considerably worse than the ``intrinsically anytime'' algorithm with weaker theoretical guarantees; and especially, such a performance gap widens with each restart. 


In addition, we would argue that the doubling trick is aesthetically unsatisfactory. Restarting wastes data and causes large ``jumps'' in the decision sequence, which can be undesirable. Therefore, although the doubling trick is theoretically sufficient in many classical settings of online learning (e.g., making fixed-learning-rate OGD anytime), more streamlined solutions are typically favored (e.g., time-varying learning rates). 

\section{Basics of the potential function}\label{section:potential_fact}

For the two potential functions defined in Section~\ref{section:discretization}, we compute their derivatives as follows. Starting from $\phi$,
\begin{equation*}
\partial_{1} \phi(x,y)=-\frac{\eps\sqrt{\alpha}}{2\sqrt{x}}\exp\rpar{\frac{y^2}{4\alpha x}},
\end{equation*}
\begin{equation*}
\partial_2 \phi(x,y)=\eps\erfi\rpar{\frac{y}{\sqrt{4\alpha x}}},
\end{equation*}
\begin{equation*}
\partial_{11} \phi(x,y)=\frac{\eps\sqrt{\alpha}}{4x^{3/2}}\rpar{\frac{y^2}{2\alpha x}+1}\exp\rpar{\frac{y^2}{4\alpha x}},
\end{equation*}
\begin{equation*}
\partial_{12}\phi(x,y)=-\frac{\eps y}{4\sqrt{\alpha}x^{3/2}}\exp\rpar{\frac{y^2}{4\alpha x}},
\end{equation*}
\begin{equation*}
\partial_{22}\phi(x,y)=\frac{\eps}{2\sqrt{\alpha x}}\exp\rpar{\frac{y^2}{4\alpha x}}.
\end{equation*}
Due to the change of variables, the derivatives of $\Phi_t$ can be concisely represented by the derivatives of $\phi$.
\begin{equation*}
\partial_1\Phi_t(V,S)=\partial_1\phi(V+z_t+k_tS,S),
\end{equation*}
\begin{equation*}
\partial_2\Phi_t(V,S)=k_t\partial_1\phi(V+z_t+k_tS,S)+\partial_2\phi(V+z_t+k_tS,S),
\end{equation*}
\begin{equation*}
\partial_{11}\Phi_t(V,S)=\partial_{11}\phi(V+z_t+k_tS,S),
\end{equation*}
\begin{equation*}
\partial_{12}\Phi_t(V,S)=k_t\partial_{11}\phi(V+z_t+k_tS,S)+\partial_{12}\phi(V+z_t+k_tS,S),
\end{equation*}
\begin{equation*}
\partial_{22}\Phi_t(V,S)=k_t^2\partial_{11}\phi(V+z_t+k_tS,S)+2k_t\partial_{12}\phi(V+z_t+k_tS,S)+\partial_{22}\phi(V+z_t+k_tS,S).
\end{equation*}

Next, we present two simple lemmas on the potential function $\Phi_t$. The first lemma shows $\Phi_t$ is convex in its second argument. The second lemma shows the negativity of $\partial_2\Phi_t(V,S)$ when $S\leq 0$. 

\begin{lemma}[Convexity]\label{lemma:convexity}
If $\eps,\alpha,k_t>0$ and $z_t>k_th_t$, then the potential function $\Phi_t(V,S)$ satisfies $\partial_{22}\Phi_t(V,S)>0$ for all $V\geq 0$ and $S\geq -h_t$.
\end{lemma}

\begin{proof}[Proof of Lemma~\ref{lemma:convexity}]
Let us drop all the subscript $t$ and let $G=h_t$. Define the shorthands $x=V+z+kS$ and $y=S$. For all $V\geq 0$ and $S\geq -G$, we have $x>0$, therefore
\begin{align*}
\partial_{22}\Phi(V,S)&=k^2\partial_{11}\phi(x,y)+2k\partial_{12}\phi(x,y)+\partial_{22}\phi(x,y)\\
&=\frac{\eps\sqrt{\alpha}}{4x^{3/2}}\exp\rpar{\frac{y^2}{4\alpha x}}\rpar{\frac{k^2y^2}{2\alpha x}+k^2-\frac{2ky}{\alpha}+\frac{2x}{\alpha}}\\
&=\frac{\eps\sqrt{\alpha}}{4x^{3/2}}\exp\rpar{\frac{y^2}{4\alpha x}}\rpar{\frac{k^2y^2}{2\alpha x}+k^2+\frac{2(V+z)}{\alpha}}\\
&> 0.\qedhere
\end{align*}
\end{proof}

\begin{lemma}[The sign of prediction]\label{lemma:neg_pred}
If $\eps,\alpha,k_t>0$ and $z_t>k_th_t$, then the potential function $\Phi_t(V,S)$ satisfies $\partial_2\Phi_t(V,S)<0$ for all $V\geq 0$ and $-h_t\leq S\leq 0$.
\end{lemma}

\begin{proof}[Proof of Lemma~\ref{lemma:neg_pred}]
Same as before, drop all the subscript $t$. Let us check $\partial_2\Phi(V,0)< 0$. Indeed, 
\begin{equation*}
\partial_2\Phi(V,0)=k\partial_1\phi(V+z,0)+\partial_2\phi(V+z,0)=-\frac{\eps k\sqrt{\alpha}}{2\sqrt{V+z}}< 0.
\end{equation*}
Moreover, $\Phi(V,S)$ is convex with respect to its second argument, due to Lemma~\ref{lemma:convexity}.
\end{proof}

The final two basic lemmas concern the property of the erfi function.

\begin{lemma}\label{lemma:erfi_1}
For all $x\geq 1$, $\erfi(x)\geq \exp(x^2)/2x$.
\end{lemma}

\begin{proof}[Proof of Lemma~\ref{lemma:erfi_1}]
Let $f(x)=\erfi(x)-\exp(x^2)/2x$. $f(1)=\erfi(1)-e/2>0$. For all $x\geq 1$,
\begin{equation*}
f'(x)=\frac{1}{2x^2}\exp(x^2)>0.\qedhere
\end{equation*}
\end{proof}

\begin{lemma}\label{lemma:erfi_2}
For all $x\geq 0$, $\erfi^{-1}(x)\leq 1+\sqrt{\log(x+1)}$.
\end{lemma}

\begin{proof}[Proof of Lemma~\ref{lemma:erfi_2}]
We first show $\erfi(x)\geq \exp(x^2-x)-1$ for all $x\geq 0$. Let $f(x)=\erfi(x)-\exp(x^2-x)+1$, then $f(0)=0$, 
\begin{equation*}
f'(x)=\exp(x^2-x)\cdot\spar{\exp(x)-(2x-1)}.
\end{equation*}
It is easy to verify $\exp(x)\geq(2x-1)$ for all $x\geq 0$.

Next, for any $y\geq 0$ let $x=\erfi^{-1}(y)$, which is also nonnegative. From the first step,
\begin{equation*}
y=\erfi(x)\geq \exp\spar{\rpar{x-\frac{1}{2}}^2-\frac{1}{4}}-1,
\end{equation*}
therefore
\begin{equation*}
x\leq \frac{1}{2}+\sqrt{\frac{1}{4}+\log(y+1)}\leq 1+\sqrt{\log(y+1)}.
\end{equation*}
Substituting $x=\erfi^{-1}(y)$ completes the proof.
\end{proof}

\section{Analysis of the base algorithm}\label{section:base_proof}

The key step of our analysis is the following potential verification argument.

\onestep*

\begin{proof}[Proof of Lemma~\ref{lemma:one_step}]
The first, preparatory step is to show $\Phi_t(V,S)\leq \Phi_{t-1}(V,S)$ for all $V\geq 0$ and $S\geq-h_{t-1}$. To see this, notice that
\begin{equation*}
\Phi_t(V,S)=\phi(V+z_t+k_tS,S),
\end{equation*}
\begin{equation*}
\Phi_{t-1}(V,S)=\phi(V+z_{t-1}+k_{t-1}S,S).
\end{equation*}
From Appendix~\ref{section:potential_fact}, $\partial_{1} \phi(x,y)\leq 0$ for all $x> 0$ and $y\in\R$. Furthermore, compare the first argument on the above right hand sides, 
\begin{align*}
V+z_t+k_tS&=V+h_t\rpar{\frac{12\alpha +4}{2\alpha-1}h_t+2S}\\
&\geq V+h_{t-1}\rpar{\frac{12\alpha +4}{2\alpha-1}h_{t-1}+2S}\tag{$S\geq -h_{t-1}$ and $h_t\geq h_{t-1}$}\\
&=V+z_{t-1}+k_{t-1}S.
\end{align*}
Concluding this argument, we have shown $\Phi_t(V,S)\leq \Phi_{t-1}(V,S)$. It means to prove the present lemma, it suffices to show
\begin{equation*}
\Phi_t(V+c^2,S+c)-\Phi_{t}(V,S)-c\partial_2\Phi_t(V,S)\leq 0,
\end{equation*}
for all $V\geq 0$, $S\geq -h_{t}$ and $c\in[-h_t-\min(S,0),h_t]$. Since all the subscripts are the same $t$, let us drop them completely to simplify the exposition. Also, let us denote $h_t=G$, which hopefully makes the ``unit'' clearer. 

Now, let us view our remaining objective as a function of $c$,
\begin{equation*}
f_{V,S}(c)\defeq\Phi(V+c^2,S+c)-\Phi(V,S)-c\partial_2\Phi(V,S).
\end{equation*}
Taking the derivatives,
\begin{equation*}
f'_{V,S}(c)=2c\partial_1\Phi(V+c^2,S+c)+\partial_2\Phi(V+c^2,S+c)-\partial_2\Phi(V,S),
\end{equation*}
\begin{align}
f''_{V,S}(c)&=2\partial_1\Phi(V+c^2,S+c)+4c^2\partial_{11}\Phi(V+c^2,S+c)+4c\partial_{12}\Phi(V+c^2,S+c)+\partial_{22}\Phi(V+c^2,S+c)\nonumber\\
&\leq 2\partial_1\Phi(V+c^2,S+c)+4G^2\partial_{11}\Phi(V+c^2,S+c)+4G\abs{\partial_{12}\Phi(V+c^2,S+c)}+\partial_{22}\Phi(V+c^2,S+c).\label{eq:upper_f}
\end{align}
Note that $f_{V,S}(0)=f'_{V,S}(0)=0$. Therefore, to prove $f_{V,S}(c)\leq 0$, it suffices to show $f''_{V,S}(c)\leq 0$ for all considered values of $V$, $S$ and $c$. Also notice that the RHS of Eq.(\ref{eq:upper_f}) has a striking similarity to the Backward Heat Equation -- in fact, after a change of variable, the resulting expressions, namely Eq.(\ref{eq:upper_f_case1}) and Eq.(\ref{eq:upper_f_case2}) below, will resemble the standard BHE on $\phi$ ($\partial_1\phi+\half\partial_{22}\phi=0$) plus ``lower order terms''. The main goal of this proof is to control such lower order terms by properly choosing $k$ and $z$.

Concretely, due to the absolute value in Eq.($\ref{eq:upper_f}$), we will analyze two cases. Technically, the first case is harder, therefore we pick $k$ to simplify its analysis. The second case is relatively easier to handle. 

\paragraph{Case 1: $\partial_{12}\Phi(V+c^2,S+c)\leq 0$.} Substituting the derivatives of $\Phi$ by the derivatives of $\phi$, we have
\begin{equation}
f''_{V,S}(c)\leq 2\partial_1\phi+(k-2G)^2\partial_{11}\phi+2(k-2G)\partial_{12}\phi+\partial_{22}\phi\Big|_{(V+c^2+z+k(S+c),S+c)}.\label{eq:upper_f_case1}
\end{equation}
The RHS means we evaluate all the derivative functions at $(V+c^2+z+k(S+c),S+c)$. Plugging in our specific choice of $k$ and $\alpha$,
\begin{align*}
f''_{V,S}(c)&\leq 2\partial_1\phi+\partial_{22}\phi\Big|_{(V+c^2+z+2(S+c),S+c)}\tag{$k=2G$}\\
&\leq 2\rpar{\partial_1\phi+\alpha\partial_{22}\phi}\Big|_{(V+c^2+z+2(S+c),S+c)}\tag{$\alpha>\half$ and Lemma~\ref{lemma:convexity}}\\
&=0.\tag{$\phi$ satisfies the BHE with constant $\alpha$}
\end{align*}

\paragraph{Case 2: $\partial_{12}\Phi(V+c^2,S+c)\geq 0$.} Similar to the first case, 
\begin{equation}
f''_{V,S}(c)\leq 2\partial_1\phi+(k+2G)^2\partial_{11}\phi+2(k+2G)\partial_{12}\phi+\partial_{22}\phi\Big|_{(V+c^2+z+k(S+c),S+c)}.\label{eq:upper_f_case2}
\end{equation}
Consider the $k$-dependent terms in Eq.(\ref{eq:upper_f_case2}), $(k+2G)^2\partial_{11}\phi+2(k+2G)\partial_{12}\phi$. Our goal is to upper bound it by $(2\alpha-1)\partial_{22}\phi$, such that in total, the RHS of Eq.(\ref{eq:upper_f_case2}) becomes $2(\partial_1\phi+\alpha\partial_{22}\phi)$, which equals 0 due to the BHE. 

Plugging in the derivatives of $\phi$ from Appendix~\ref{section:potential_fact}, for all inputs $(x,y)$,
\begin{align*}
&(k+2G)^2\partial_{11}\phi+2(k+2G)\partial_{12}\phi-(2\alpha-1)\partial_{22}\phi\Big|_{(x,y)}\\
=~&\frac{\eps}{4\sqrt{\alpha}x^{3/2}}\exp\rpar{\frac{y^2}{4\alpha x}}\spar{(k+2G)^2\rpar{\frac{y^2}{2x}+\alpha}-2(k+2G)y-2x(2\alpha-1)}\\
=~&\frac{\eps}{4\sqrt{\alpha}x^{3/2}}\exp\rpar{\frac{y^2}{4\alpha x}}\spar{16G^2\rpar{\frac{y^2}{2x}+\alpha}-8Gy-2x(2\alpha-1)}
\end{align*}
We aim to show the bracket on the RHS is negative at $x=V+c^2+z+k(S+c)$ and $y=S+c$, where $k=2G$. This amounts to showing
\begin{equation*}
\Diamond\defeq\frac{4G^2(S+c)^2}{V+c^2+z+2G(S+c)}+8\alpha G^2-(4\alpha G+2G)(S+c)-(2\alpha -1)(V+c^2+z)\leq 0.
\end{equation*}

The idea is that we can pick a large enough $z$ to make it hold. Concretely, 
\begin{itemize}
\item If $S+c>0$, then since $\alpha>\half$,
\begin{align*}
\Diamond&\leq \frac{4G^2(S+c)^2}{2G(S+c)}+8\alpha G^2-(4\alpha G+2G)(S+c)-(2\alpha -1)z\\
&=-4\alpha G(S+c)+8\alpha G^2-(2\alpha -1)z\\
&\leq 8\alpha G^2-(2\alpha -1)z.
\end{align*}
and it suffices to pick
\begin{equation*}
z\geq \frac{8\alpha}{2\alpha-1}G^2.
\end{equation*}
\item If $S+c\leq 0$, then since we require $c\geq -G-\min(S,0)$, we also have $S+c\geq \min(S,0)+c\geq -G$. As long as $z\geq 4G^2$, 
\begin{align*}
\Diamond&\leq \frac{4G^4}{z-2G^2}+12\alpha G^2+2G^2-(2\alpha-1)z\\
&\leq 12\alpha G^2+4G^2-(2\alpha-1)z.
\end{align*}
It suffices to pick
\begin{equation*}
z\geq \frac{12\alpha +4}{2\alpha-1}G^2.
\end{equation*}
\end{itemize}
In summary, $z=\frac{12\alpha +4}{2\alpha-1}G^2$ ensures $\Diamond\leq 0$. Due to the BHE on $\phi$,
\begin{equation*}
f''_{V,S}(c)\leq 2(\partial_1\phi+\alpha\partial_{22}\phi)\Big|_{(V+c^2+z+2G(S+c),S+c)}= 0.
\end{equation*}

Combining the two cases completes the proof.
\end{proof}

The following lemma characterizes the Fenchel conjugate of $\Phi_t$ (with respect to its second argument). 

\begin{lemma}[Conjugate]\label{lemma:conjugate}
With $\eps,\alpha,k_t>0$ and $z_t>k_th_t$, for all $u\geq 0$,
\begin{equation*}
\sup_{S\in[-h_t,\infty)}\spar{uS-\Phi_t(V,S)}\leq u\bar S+\eps\sqrt{\alpha (V+z_t+k_t\bar S)},
\end{equation*}
where
\begin{equation*}
\bar S= 4\alpha k_t\rpar{1+\sqrt{\log(2u\eps^{-1}+1)}}^2+\sqrt{4\alpha (V+z_t)}\rpar{1+\sqrt{\log(2u\eps^{-1}+1)}}.
\end{equation*}
\end{lemma}

\begin{proof}[Proof of Lemma~\ref{lemma:conjugate}]
Throughout this proof, we drop all the subscript $t$ and write $G$ in place of $h_t$. 

We first show that the supremum over $S$ in the Fenchel conjugate is attainable by some $S^*\in[0,\infty)$. To this end, define the function $f(S)\defeq uS-\Phi(V,S)$. $f$ is continuous, with $f'(S)=u-\partial_2\Phi(V,S)$. Moreover, due to Lemma~\ref{lemma:convexity}, $f$ is concave on $[-G,\infty)$. The existence of $S^*\in[0,\infty)$ then follows from analyzing the boundary.
\begin{itemize}
\item For all $S\in[-G,0]$, we have $f'(S)\geq 0$. The reason is $u\geq 0$, and $\partial_2\Phi(V,S)\leq 0$ due to Lemma~\ref{lemma:neg_pred}.
\item For sufficiently large $S$, we aim to show $f'(S)<0$. Let us begin by writing down $\partial_2\Phi(V,S)$, from Appendix~\ref{section:potential_fact}.
\begin{equation*}
\partial_2\Phi(V,S)=\eps\erfi\rpar{\frac{S}{\sqrt{4\alpha (V+z+kS)}}}-\frac{\eps k\sqrt{\alpha}}{2\sqrt{V+z+kS}}\exp\rpar{\frac{S^2}{4\alpha (V+z+kS)}}.
\end{equation*}
Now consider large $S$ that satisfies $S\geq \sqrt{4\alpha(V+z+kS)}$. Due to an estimate of the erfi function (Lemma~\ref{lemma:erfi_1}),
\begin{equation*}
\erfi\rpar{\frac{S}{\sqrt{4\alpha (V+z+kS)}}}\geq \frac{\sqrt{\alpha(V+z+kS)}}{S}\exp\rpar{\frac{S^2}{4\alpha (V+z+kS)}}.
\end{equation*}
Moreover,
\begin{equation*}
\frac{\sqrt{\alpha(V+z+kS)}}{S}-\frac{k\sqrt{\alpha}}{\sqrt{V+z+kS}}=\frac{\sqrt{\alpha}(V+z)}{S\sqrt{V+z+kS}}\geq 0.
\end{equation*}
Therefore,
\begin{align}
\partial_2\Phi(V,S)&=\spar{\frac{\eps}{2}\erfi\rpar{\frac{S}{\sqrt{4\alpha (V+z+kS)}}}-\frac{\eps k\sqrt{\alpha}}{2\sqrt{V+z+kS}}\exp\rpar{\frac{S^2}{4\alpha (V+z+kS)}}}\nonumber\\
&\hspace{200pt} +\frac{\eps}{2}\erfi\rpar{\frac{S}{\sqrt{4\alpha (V+z+kS)}}}\nonumber\\
&\geq \frac{\eps}{2}\erfi\rpar{\frac{S}{\sqrt{4\alpha (V+z+kS)}}}.\label{eq:prediction_lower}
\end{align}
For sufficiently large $S$, we have $\rhs> u$, hence $f'(S)<0$.
\end{itemize}

Summarizing the above, we have shown that there exists $S^*\in[0,\infty)$ such that
\begin{equation*}
\Phi^*_V(u)\defeq \sup_{S\in[-G,\infty)}uS-\Phi(V,S)=uS^*-\Phi(V,S^*). 
\end{equation*}
Moreover, $S^*$ should satisfy the first order optimality condition $f'(S^*)=0$, i.e., $u=\partial_2\Phi(V,S^*)$. Our goal next is to upper bound $S^*$ by a function of $u$. Again, we analyze two cases.

\paragraph{Case 1.} If $S^*$ satisfies $S^*\leq\sqrt{4\alpha(V+z+kS^*)}$, then by taking the square on both sides and regrouping the terms, we have $(S^*)^2-4\alpha kS^*-4\alpha(V+z)<0$. Solving this quadratic inequality, 
\begin{align*}
S^*&\leq \frac{1}{2}\spar{4\alpha k+\sqrt{(4\alpha k)^2+16\alpha (V+z)}}\\
&= 2\alpha k+\sqrt{4\alpha^2 k^2+4\alpha (V+z)}\\
&\leq 4\alpha k+\sqrt{4\alpha(V+z)}.
\end{align*}

\paragraph{Case 2.} If $S^*$ satisfies $S^*\geq\sqrt{4\alpha(V+z+kS^*)}$, then same as the earlier analysis in the present proof, Eq.(\ref{eq:prediction_lower}), we have
\begin{equation*}
u\geq \frac{\eps}{2}\erfi\rpar{\frac{S^*}{\sqrt{4\alpha (V+z+kS^*)}}}.
\end{equation*}
For conciseness, define the notation $p=\erfi^{-1}(2u\eps^{-1})$. Then, $(S^*)^2-4\alpha kp^2S^*-4\alpha p^2(V+z)\leq 0$. Solving the quadratic inequality,
\begin{align*}
S^*&\leq \frac{1}{2}\spar{4\alpha kp^2+\sqrt{(4\alpha kp^2)^2+16\alpha p^2(V+z)}}\\
&\leq 2\alpha kp^2+\sqrt{4\alpha^2 k^2p^4+4\alpha p^2(V+z)}\\
&\leq 4\alpha kp^2+\sqrt{4\alpha (V+z)}p.
\end{align*}

Now we can combine the above two cases. Specifically, $p\leq1+\sqrt{\log(2u\eps^{-1}+1)}$ due to Lemma~\ref{lemma:erfi_2}. Therefore,
\begin{equation*}
S^*\leq 4\alpha k\rpar{1+\sqrt{\log(2u\eps^{-1}+1)}}^2+\sqrt{4\alpha (V+z)}\rpar{1+\sqrt{\log(2u\eps^{-1}+1)}}.
\end{equation*}
Define the RHS as $\bar S$. Then, from the definition of the Fenchel conjugate, 
\begin{align*}
\sup_{S\in[-G,\infty)}\spar{uS-\Phi(V,S)}&=uS^*-\Phi(V,S^*)\\
&=uS^*-\eps\sqrt{\alpha (V+z+kS^*)}\spar{2\int_0^{\frac{S^*}{\sqrt{4\alpha (V+z+kS^*)}}}\erfi(u)du-1}\\
&\leq u\bar S+\eps\sqrt{\alpha (V+z+k\bar S)}.
\end{align*}
Plugging in $\bar S$ completes the proof.
\end{proof}

Combining the previous two lemmas, the following lemma characterizes the regret of the base algorithm. 

\begin{lemma}[Regret of the base algorithm]\label{lemma:base}
With $\eps>0$, $\alpha>\half$, $k_t=2h_t$ and $z_t=\frac{12\alpha +4}{2\alpha-1}h_t^2$, Algorithm~\ref{alg:base} guarantees for all $T\in\N_+$ and $\tilde u\geq 0$, 
\begin{equation*}
\sum_{t=1}^T\tilde l_t(\tilde y_t-\tilde u)\leq \eps\sqrt{\alpha\rpar{\tilde V_T+z_T+k_T\bar S}}+\tilde u\bar S,
\end{equation*}
where
\begin{equation*}
\bar S= 4\alpha k_T\rpar{1+\sqrt{\log(2\tilde u\eps^{-1}+1)}}^2+\sqrt{4\alpha \rpar{\tilde V_T+z_T}}\rpar{1+\sqrt{\log(2\tilde u\eps^{-1}+1)}}.
\end{equation*}
\end{lemma}

\begin{proof}[Proof of Lemma~\ref{lemma:base}]
First, we can obtain a cumulative loss upper bound by simply summing the one-step potential verification bound (Lemma~\ref{lemma:one_step}). Letting $c=-\tilde l_t$, $V=\tilde V_{t-1}$ and $S=\tilde S_{t-1}$ in Lemma~\ref{lemma:one_step},
\begin{equation*}
\tilde l_t\tilde y_t\leq \Phi_{t-1}\rpar{\tilde V_{t-1},\tilde S_{t-1}}-\Phi_t\rpar{\tilde V_{t},\tilde S_{t}},
\end{equation*}
\begin{equation*}
\sum_{t=1}^T\tilde l_t\tilde y_t\leq \Phi_0\rpar{0,0}-\Phi_T\rpar{\tilde V_{T},\tilde S_{T}}.
\end{equation*}

Then, the proof follows from a loss-regret duality and the Fenchel conjugate computation from Lemma~\ref{lemma:conjugate}.
\begin{align*}
\sum_{t=1}^T\tilde l_t(\tilde y_t-\tilde u)&\leq \tilde S_T\tilde u+\Phi_0(0,0)-\Phi_T\rpar{\tilde V_T,\tilde S_T}\\
&\leq \Phi_0(0,0)+\sup_{S\in[-h_T,\infty)}\spar{S\tilde u-\Phi_T\rpar{\tilde V_T,S}}\\
&\leq-\eps\sqrt{\alpha\cdot \frac{12\alpha +4}{2\alpha-1}h_0^2}+\eps\sqrt{\alpha\rpar{\tilde V_T+z_T+k_T\bar S}}+\tilde u\bar S\\
&\leq \eps\sqrt{\alpha\rpar{\tilde V_T+z_T+k_T\bar S}}+\tilde u\bar S.
\end{align*}
Plugging in $\bar S$ from Lemma~\ref{lemma:conjugate} completes the proof.
\end{proof}

\section{Analysis of the meta algorithm}\label{section:appendix_meta}

First, we prove Lemma~\ref{lemma:well_pose}, which certifies the well-posedness of our algorithm.

\well*

\begin{proof}[Proof of Lemma~\ref{lemma:well_pose}]
First, notice that $|\tilde l_t|\leq \abs{l_t}=\abs{\inner{g_t}{w_t}}\leq h_t$. 

Next, we prove by induction. Consider $-\sum_{i=1}^{t-1}\tilde l_i$, which is defined as $\tilde S_{t-1}$ in the base algorithm (Algorithm~\ref{alg:base}). Suppose $\tilde S_{t-1}\geq-h_{t-1}$, which trivially holds at $t=1$. Let us analyze two cases.
\begin{itemize}
\item If $\tilde S_{t-1}\geq 0$, then $-\sum_{i=1}^t\tilde l_i=\tilde S_{t-1}-\tilde l_t\geq \tilde S_{t-1}-|\tilde l_t|\geq -h_t$.
\item If $-h_{t-1}\leq \tilde S_{t-1}<0$, then due to Lemma~\ref{lemma:neg_pred}, the prediction $\tilde y_t$ of the base algorithm satisfies $\tilde y_t<0$. The meta algorithm projects it to $y_t=0$. Then, due to our definition of $\tilde l_t$ in the meta algorithm,
\begin{equation*}
\tilde l_t=\begin{cases}
l_t,& l_t\leq 0,\\
0,& \textrm{else},
\end{cases}
\end{equation*}
which is non-positive. Therefore, $-\sum_{i=1}^t\tilde l_i=\tilde S_{t-1}-\tilde l_t\geq \tilde S_{t-1}\geq -h_{t-1}\geq -h_t$.
\end{itemize}
An induction completes the proof.
\end{proof}

Next, we prove our main result.

\meta*

\begin{proof}[Proof of Theorem~\ref{thm:meta}]
Since the meta algorithm simply applies two existing black-box reductions \cite{cutkosky2018black,cutkosky2020parameter}, the proof is straightforward given Lemma~\ref{lemma:base}, the regret bound of the 1D base algorithm. First, due to a polar decomposition theorem \cite[Theorem~2]{cutkosky2018black}, the regret can be decomposed into the regret of $\A_B$ with respect to $u/\norm{u}$, plus the regret of $y_t$ with respect to $\norm{u}$. Then, the latter is upper-bounded by the regret of $\tilde y_t$ evaluated on the surrogate losses $\tilde l_t$ -- this is because our definition of $y_t$ and $\tilde l_t$ follows the procedure of \cite[Theorem~2]{cutkosky2020parameter}, where a convex constraint can be added to an unconstrained algorithm without changing its regret bound. In summary, we have
\begin{align*}
\reg_T(Env,u)&= \sum_{t=1}^Tl_t(y_t-\norm{u})+\norm{u}\sum_{t=1}^T\inner{g_t}{w_t-u/\norm{u}}\\
&\leq \sum_{t=1}^T\tilde l_t(\tilde y_t-\norm{u})+\norm{u}\sum_{t=1}^T\inner{g_t}{w_t-u/\norm{u}}.
\end{align*}
The two regret terms on the RHS represent the regret bound of $\A_{1d}$ and $\A_\ball$, respectively. 

Now, the first term follows from Lemma~\ref{lemma:base}, where $\tilde V_T=\sum_{t=1}^T\tilde l^2_t\leq \sum_{t=1}^Tl^2_t=\sum_{t=1}^T\inner{g_t}{w_t}^2\leq V_T$. As for the regret of $\A_\ball$, due to \cite[Theorem 4.14]{orabona2023modern},
\begin{equation*}
\sum_{t=1}^T\inner{g_t}{w_t-u/\norm{u}}\leq 2\sqrt{2V_T}.
\end{equation*}
Combining these two components completes the proof.
\end{proof}

Finally, let us use asymptotic orders to make this bound a bit more interpretable. Consider the regime of large $\norm{u}$ and $V_T$, i.e., $\norm{u}\gg \eps$ and $V_T\gg h_T^2$. We preserve the dependence on $\eps^{-1}$, as $\eps$ is a ``confidence hyperparameter'' that can get arbitrarily close to $0$. In contrast, $\alpha$ is a moderate constant slightly larger than $\half$, therefore we subsume it by the big-Oh. 

Using $\log(1+x)\leq x$, we can crudely bound $\bar S$ in Theorem~\ref{thm:meta} by
\begin{align*}
\bar S&\leq 8\alpha h_T\rpar{1+\sqrt{2\norm{u}\eps^{-1}}}^2+\sqrt{4\alpha(V_T+z_T)}\rpar{1+\sqrt{2\norm{u}\eps^{-1}}}\\
&=8\alpha h_T+\sqrt{4\alpha (V_T+z_T)}+o\rpar{\norm{u}\eps^{-1}\sqrt{V_T}}.
\end{align*}
Plugging this crude bound of $\bar S$ into the first term of the regret bound, we have
\begin{align*}
\reg_T(Env,u)&\leq \eps\sqrt{\alpha\rpar{V_T+z_T+16\alpha h_T^2+2h_T\sqrt{4\alpha(V_T+z_T)}}}+\eps\sqrt{o\rpar{\norm{u}\eps^{-1}\sqrt{V_T}}}+\norm{u}\rpar{\bar S+2\sqrt{2V_T}}\\
&\leq \eps\sqrt{\alpha}\rpar{\sqrt{V_T+z_T}+4\sqrt{\alpha}h_T}+o\rpar{\norm{u}V_T^{1/4}}+\norm{u}\rpar{\bar S+2\sqrt{2V_T}}.
\end{align*}

Next, notice that $\bar S=O\rpar{\sqrt{V_T\log(\norm{u}\eps^{-1})}\vee h_T\log(\norm{u}\eps^{-1})}$. Using it to replace the remaining $\bar S$ above,
\begin{equation*}
\reg_T(Env,u)\leq \eps\rpar{\sqrt{\alpha(V_T+z_T)}+4\alpha h_T}+o\rpar{\norm{u}V_T^{1/4}}+\norm{u} O\rpar{\sqrt{V_T\log(\norm{u}\eps^{-1})}\vee h_T\log(\norm{u}\eps^{-1})}.
\end{equation*}
The second term can be assimilated into the third term. The result becomes
\begin{equation}\label{eq:regret_simple}
\reg_T(Env,u)\leq \eps\rpar{\sqrt{\alpha\rpar{V_T+z_T}}+4\alpha h_T}+\norm{u} O\rpar{\sqrt{V_T\log(\norm{u}\eps^{-1})}\vee h_T\log(\norm{u}\eps^{-1})}.
\end{equation}

Note that this bound is not only valid for large $\norm{u}$, but also valid when $u=0$ (this can be directly verified from Theorem~\ref{thm:meta}). Therefore, it characterizes the loss-regret tradeoff. 

\end{document}